\documentclass[letterpaper, 10 pt, conference]{ieeeconf}  

\IEEEoverridecommandlockouts                              

\usepackage{graphicx} 
\usepackage{algorithm}
\usepackage{algpseudocode}
\usepackage{amsmath} 
\usepackage{amssymb}  
\usepackage{cite}
\newtheorem{theorem}{Theorem}[section]
\newtheorem{lemma}[theorem]{Lemma}
\usepackage{url}

\title{\LARGE \bf
Time-Ordered Ad-hoc Resource Sharing for
Independent Robotic Agents
}

\author{Arjo Chakravarty$^{1}$, Michael X. Grey$^{2}$, M. A. Viraj J. Muthugala$^{3}$ and Mohan Rajesh Elara$^{3}$ 
\thanks{This research is supported by Intrinsic llc., the National Robotics Programme under its National Robotics Programme (NRP) BAU, Ermine III: Deployable Reconfigurable Robots, Award No. M22NBK0054,  A*STAR under its “RIE2025 IAF-PP Advanced ROS2-native Platform Technologies for Cross sectorial Robotics Adoption (M21K1a0104)” programme, and also supported by SUTD Growth Plan (SGP) Grant, Grant Ref. No. PIE-SGP-DZ-2023-01.}
\thanks{$^{1}$Arjo Chakravarty is with Intrinsic LLC and Singapore University of Technology and Design, Singapore. 
        {\tt\small arjoc@intrinsic.ai, arjo\_chakravarty@mymail.sutd.edu.sg}}%
\thanks{$^{2}$Michael X. Grey is with Intrinsic LLC
        {\tt\small mxgrey@intrinsic.ai}}%
\thanks{$^{3}$M. A. Viraj J. Muthugala and Mohan Rajesh Elara are with Singapore University of Technology and Design, Singapore.
        {\tt\small \{viraj\_jagathpriya,rajeshelara\}@sutd.edu.sg}}%
}

\begin{document}

\maketitle
\thispagestyle{empty}
\pagestyle{empty}

\begin{abstract}

Resource sharing is a crucial part of a multi-robot
system. We propose a Boolean satisfiability problem (SAT) based approach to resource sharing.
Our key contributions are an algorithm for converting any
constrained assignment to a weighted-SAT based optimization.
We propose a theorem that allows optimal resource assignment problems to be solved via repeated application of a SAT solver.
Additionally we show a way to encode continuous time
ordering constraints using Conjunctive Normal Form (CNF). We
benchmark our new algorithms and show that they can be used
in an ad-hoc setting. We test our algorithms on a
fleet of simulated and real world robots and show that the
algorithms are able to handle real world situations. Our algorithms
and test harnesses are open source and build on Open-RMF’s fleet
management system.

\end{abstract}

\section{INTRODUCTION}

Multi-robot systems is a well studied field with a variety of sub problems. There has been extensive study of robot communications, joint exploration strategies \cite{9476870}, multi-agent path finding \cite{SHARON201540} \cite{kottinger2022conflict} and task allocation\cite{Gerkey02}. Open-RMF is an open source framework that facilitates multi-fleet orchestration \cite{OpenRMF}. It comes with tools for handling problems such as traffic management, map alignment and optimal task allocation \cite{Gerkey02}. Robotics researchers have spent much effort on optimizing multi-robot path planning\cite{SHARON201540}\cite{stern2019multi}, task assignment \cite{Gerkey02}\cite{NamNPHard} and multi-robot localization\cite{yu2020review}. However, as we move these systems to production one problem that arises is ad hoc (on-demand) physical resource contention\footnote{\url{https://github.com/open-rmf/rmf/discussions/83#discussioncomment-1123844}}. Suppose two or more robots intend to use the same item like a parking spot, charger, or an elevator. There needs to be a way for their use of the resources to be orchestrated. This paper proposes algorithms for a reservation system built upon the notion of optimizing the assignment of resources while taking into account individual costs for robots using a  Boolean satisfy-ability solver (SAT). When combined with traffic deconfliction, it allows lifelong operation of a multi-robot system. Our primary contributions in this paper are:
\begin{itemize}
    \item A Conjunctive Normal Form (CNF) formulation of the robot resource allocation problem
    \item A greedy algorithm for optimizing resource assignment
    \item An algorithm for converting any assignment problem and its SAT formulation into a cost optimization problem.
    \item A CNF encoding of the scheduling problem 
    \item An open source package integrated with ROS 2 and Open-RMF that handles resource assignment 
\end{itemize}
The rest of this paper is organised as follows:  Section \ref{sec:lit_review} covers a background literature review.  Section \ref{sec:ps_1} introduces the resource allocation problem for fixed time ranges. Section \ref{sec:cost_opt} compares a weighted SAT based approach to a greedy approach. Section \ref{sec:cost_opt} proposes and compares two algorithms: one based on SAT and another greedy method.  Section  \ref{sec:flexitime} extends the CNF formulation introduced in Section \ref{sec:ps_1} to support time ranges instead of fixed time and shows that it is faster than discretization.  Finally, Section \ref{sec:val} describes an experimental validation done in simulation and on physical robots.

\section{Background\label{sec:lit_review}}

One straightforward method to solve the resource assignment problem is to check if the robot's destination is currently occupied and wait for it to become available. This overly simplistic approach easily leads to deadlock, for example, if two robots each want to go to the location currently occupied by the other, then both are stuck waiting on the other to move.

The classic algorithm for resource assignment is the Khun-Munkres method \cite{munkres1957algorithms}. This method assumes there are $n$ resources and $m$ requesters. Each requester assigns a cost for a given resource, e.g. the distance the requester would need to travel from their current location to arrive at each resource. The algorithm will converge on the optimal set of resource assignments in a matter of $O(n^3)$ time \cite{munkres1957algorithms}. This method however does not account for scenarios where it is necessary for multiple requesters to take turns with one resource as is often the case in a realistic deployment. Additionally it is not possible to to encode constraints into the Kuhn-Munkres algorithm. This makes it unsuitable for our use as we often need to encode constraints like ``This robot must have access to a charger within the next two hours''.

It can be shown that the robot resource assignment problem has equivalent complexity to the Travelling Salesman Problem (TSP) making it NP-hard \cite{NamNPHard}. Several approaches for solving such problems exist such as using Linear Programming (LP) \cite{NamNPHard} or Constraint Programming (CP). Integer linear programming (ILP) is heavily favoured by many TSP solvers \cite{cook2014pursuit}, however fast ILP solvers such as Gurobi \cite{gurobi} and C-PLEX \cite{cplex2009v12} remain proprietary. Google's open source OR-tools use constraint programming solvers for such NP-hard optimizations \cite{cpsatlp_cpaior_masterclass}.

The approach we favour is similar to the approach used by OR-tools as there are many very mature open source SAT solvers available for use \cite{balyo2023proceedings}. The Conflict Driven Clause Learning algorithm allows SAT solvers to learn clauses which are impossible enabling us to limit the search space drastically \cite{cdcl1}\cite{zhang2001efficient}\cite{MRTASAT}.  It is easy to encode new constraints using CNF as compared to LP. This means that we know if a set of requested inputs is impossible even before trying to optimize the constraints. Similar approaches have been used by others, for instance Imeson and Smith use SAT and TSP solvers to solve integrated task and motion planning \cite{MRTASAT}. However, our approach is much simpler.

A form of the SAT problem which can encode optimization is the Weighted MaxSAT problem \cite{heras2008minimaxsat}. In the MaxSAT problem there are a few hard and a few soft constraints. The hard constraints are ones which cannot be violated, while the cost of the soft constraints is minimized with weights\cite{heras2008minimaxsat}. This is not unlike the problem we have at hand---in fact the formulation we present is one that could potentially be solved using MaxSAT solvers. However, for our formulation we do not have soft constraints thus significantly reducing the complexity and eliminating the need for MaxSAT.

Often CP-SAT solvers encode the costs of individual assignments themselves as CNF formulas \cite{cpsatlp_cpaior_masterclass}. This leads to the conundrum where floating point numbers cannot be represented and only integers can be used. In most cases this is rarely a problem as we can just multiply the costs by an arbitrarily large number \cite{cpsatlp_cpaior_masterclass}. For our problem however, thanks to the fact that only one alternative per request needs to be assigned (see section \ref{sec:ps_1} for definition of alternative), we can perform some simple tricks to speed up the search and limit our search space when dealing with assignment type problems. Additionally, we show that it is possible to express such floating point constraints in terms of total orders, thus enabling conventional SAT solvers to reason about them.

\section{Problem Statement}\label{sec:ps_1} 

The problem we wish to solve is one where a group of heterogeneous robots request access to a certain set of resources. For instance, robots may need to use charging stations, parking spaces, or elevators. When a robot needs to utilize some kind of resource, it proposes a set of acceptable alternative choices along with the cost of each alternative. For example if a robot needs to charge its battery and is compatible with three different charging stations, it would request an assignment to one of the three stations, listing the cost of going to each station from its current location (e.g. time or distance to travel).

\subsection{Definitions}
Given a set of resource requests, we assign resources across those requests based on the available alternatives in a way that minimizes the overall cost without having any assignments that overlap in time.
\begin{itemize}
    \item Each request can be defined as $R_i$.
    \item \label{discrete_formulation} A request has $n$ alternatives $\alpha_{i,j} = (s_{i,j},r_{i,j},d_{i,j}, c_{i,j})$ where $ s_{i,j},d_{i,j},c_{i,j} \in \mathbb{R}$. Each alternative is made of a start time $s$, a resource $r$, a cost $c$ and a duration $d$.
    \item An alternative $\alpha_{i,j}$ conflicts with another alternative $\alpha_{k,m}$ if it shares the same resource (i.e. $r_{i,j} = r_{k,m}$) and it overlaps in time with the other alternative.
\end{itemize}

\subsection{Simplification to CNF} 
The proposed problem can be simplified into a CNF. Suppose the variable $x_{i,j}$ represents the fact that the alternative $\alpha_{i,j}$ is awarded. We have two constraints:

First only one alternative from a request $R_i$ with $n$ alternatives can be awarded. This gives us the clauses shown in \eqref{c1} and \eqref{c2}
\begin{equation}
(x_{i, 0} \lor x_{i, 1} \lor x_{i, 2} ... \lor x_{i,n}) \label{c1}
\end{equation} 
\begin{equation}
\bigwedge_{j,k} (\neg x_{i,j} \lor \neg x_{i, k}), \text{where } j < k \label{c2}
\end{equation} 
Secondly, no items with conflicts can be assigned. So for every pair of conflicts $\alpha_{i,j}$, $\alpha_{i,m}$ we get the clause \eqref{c3}
\begin{equation}
 (\neg x_{i,j} \lor \neg x_{k, m}) \label{c3}
\end{equation} 

\section{Cost Optimization\label{sec:cost_opt}}

\subsection{Weighted-SAT Based Approach}
While SAT can provide a feasible solution, it does not provide an easy way to calculate cost. Repeated application of SAT is a common technique used when dealing with such optimizations \cite{heras2008minimaxsat}. In our case we can formulate our optimization in terms of the boolean variables as shown in \eqref{eq:weighted} where  $x_{i,j} \in \{0,1\}$ and $ c_{i,j} \in \mathbb{R} $.
\begin{equation}
\arg\underset{x_{i,j}}{\min}  \sum_{i,j} x_{i,j} c_{i,j} \label{eq:weighted}
\end{equation}

One of the simplest ways to solve this problem is to ask the SAT solver for SAT assignment, then negate the assignment and add it as a clause to the original. We keep doing this until the SAT solver returns "unsatisfiable" (UNSAT). At this point we can confirm that there are no more alternative assignments. Such a naive brute-force approach does not scale very well given a scenario where there are many feasible solutions. Even when we come across the optimal solution, we would need to continue performing an exhaustive search to eliminate all possible alternatives before we can prove that the candidate solution is optimal.

There is additional information in the problem domain that we can exploit to define necessary and sufficient conditions for determining whether a feasible solution is optimal without an exhaustive search. This motivates the following Lemma:

\begin{lemma}\label{lemma:1}
Given an assignment $A_1=(x_{0,j} , x_{1,l} .... x_{n, z})$ if a cheaper solution $A_2$ exists, then $\exists x_{i,j} \in A_2$ such that given $x_{i,k} \in A_1$, $c_{i,j} < c_{i,k}$.
\end{lemma}

\begin{proof}
(by contradiction)

Suppose not. Given that there exist two solutions $A_2$ and $A_1$ where $A_2$ is cheaper than $A_1$. Then from the contradiction it follows that $\forall i, i < n, $ given $x_{i,j} \in A_2$ and $x_{i,k} \in A_1$, $c_{i,j} \geq c_{i,k}$. This leads to the total cost of the solution $cost(A_2) = c_{0,j_1}+ c_{1,j_2}+...+c_{n, j_n} \geq c_{0,k_1}+ c_{1,k_2}+...+c_{n, k_n} = cost(A_1)$. But this is a contradiction because we had defined $A_2$ to be cheaper than $A_1$. 
\end{proof}

Algorithm \ref{alg:heur} uses Lemma \ref{lemma:1} to formulate a new clause every time a solution is found. Given a SAT solution we add a clause to the formula that requires at least one of the cheaper alternatives to be true. This guides the solver towards the optimal solution faster than a brute force search. Additionally, as a result of Lemma \ref{lemma:1} we immediately know if a solution is optimal. This is because if a cheaper solution exists it must satisfy Lemma \ref{lemma:1}. Thus if we get an UNSAT solution from the SAT solver, we know that a cheaper solution cannot exist and our previous assignment is optimal.

Fig. \ref{fig:brute} shows how this heuristic greatly reduces computation time for problems. Both a Naive brute-force approach and Algorithm \ref{alg:heur} were made to solve 100 problems with a large number of feasible alternatives. The heuristic based approach is able to handle larger problems relatively well compared to the naive SAT approach which quickly becomes too slow. One of the nice features of incremental optimization methods with SAT is that we are able to quickly obtain solutions which are sub-optimal but feasible, which is useful in real world scenarios when having an effective solution quickly is more important than obtaining the optimal solution eventually.

\begin{figure}[!b]
    \centering
    \includegraphics[width=8cm]{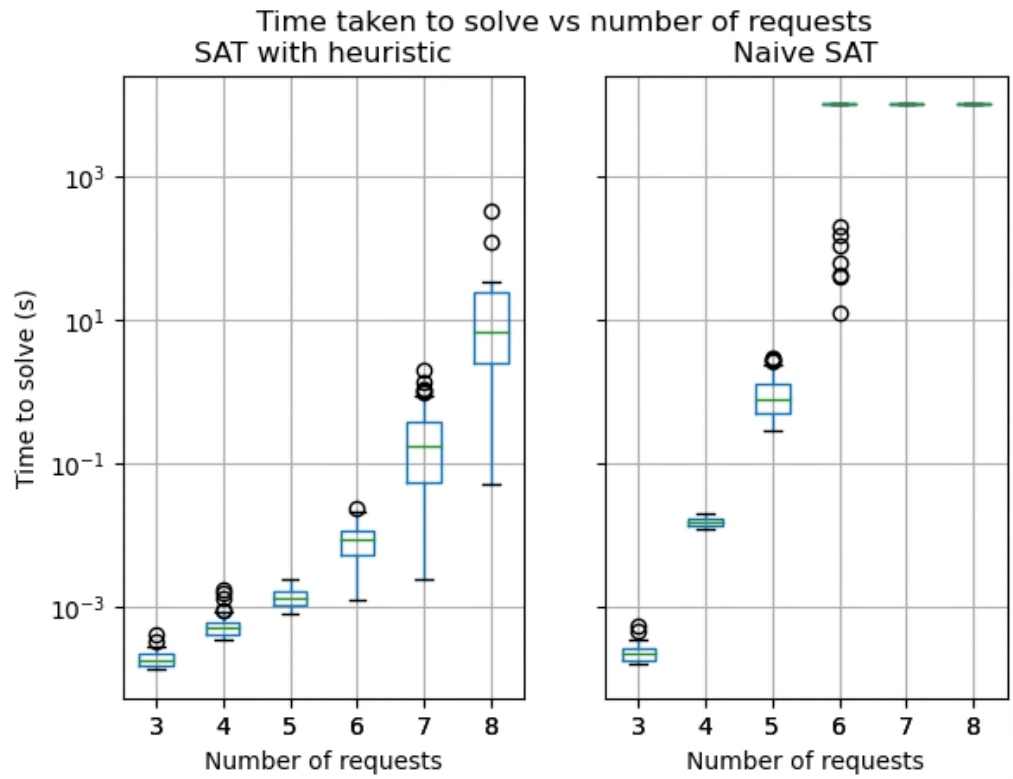}
    \caption{Box-plot of Naive SAT based search vs heuristic driven search. Both solvers were given 100 hard optimization problems of different sizes. Benchmark was run on a 24-core GCP instance with code written in Rust. Each request has 10 alternatives.}
    \label{fig:brute}
\end{figure}

\begin{algorithm}
\caption{SAT Cost Optimizer with Assignment Heuristic}\label{alg:heur}
\begin{algorithmic}
\Require An incremental SAT solver
\State $literals \gets$ Generate SAT literals from requests
\State $clauses \gets \{\} $
\State $best\_cost \gets \infty$
\State $best\_solution \gets \emptyset$
\State $OPEN \gets \{clauses\}$
\While{$OPEN$ is not empty}
\State $current\_clauses \gets$ pop an element in $OPEN$
\State $s \gets$ Call Incremental SAT with $current\_clauses$
\If{$s$ is UNSAT}
    \State continue.
\EndIf
\If{$cost(s) < best\_cost$}
\State $best\_cost \gets cost(s)$
\State $best\_solution \gets s$
\State $new\_clause \gets ()$
\State \textit{// Create clause generated by Lemma \ref{lemma:1}}
\ForAll{$x_{ij}$ in s}
    \ForAll{$cost(x_{ik}) < cost(x_{ij})$}
        \State $new\_clause \gets new\_clause \lor x_{ik}$
    \EndFor
\EndFor
\State  $new\_clauses \gets clauses \land new\_clause$
\State  $new\_clauses \gets new\_clauses \land \neg s$
\Else
\State  $new\_clauses \gets clauses \land \neg s$
\EndIf
\State $OPEN \gets new\_clauses$
\EndWhile
\end{algorithmic}
\end{algorithm}

\subsection{Alternate Solution to Optimization via a Greedy Conflict Driven Approach} \label{sec:alt_soln}

We can also find an optimal solution using a greedy search (see Algorithm \ref{alg:greedy}). The approach relies on greedily picking the lowest cost alternatives for each reservation, then when a conflict arises we branch the search based on the conflict. For instance if $\alpha_{m,n}$ and $\alpha_{j,k}$ conflict with each other we branch on the fact that either $\alpha_{m,n}$ must belong to the solution and hence eliminate $\alpha_{j,k}$, or $\alpha_{j,k}$ must belong to the solution and hence eliminate $\alpha_{m,n}$, or neither belong to the solution so both should be eliminated.

This gives us a branching factor of 3 where the frequency of branching comes from the number of conflicting alternatives. Therefore the more conflicts there are, the longer the greedy search will take. Additionally, if a solution is unsatisfied, the optimizer will examine every possible branch before it terminates. This combinatorial approach will not scale well with an increasing number of conflicting requests and alternatives.

\begin{algorithm}
\caption{Greedy Conflict-Driven Optimizer}\label{alg:greedy}
\begin{algorithmic}

\State $B_0\gets ()$
\State $S_0 = $ lowest cost assignment $a_{i,j}$ $\forall$ requests $R_i$
\State $OPEN \gets \{(S_0, B_0)\}$
\While{$OPEN$ is not empty}
\State $(solution, B_0) \gets$ pop lowest cost assignment in $OPEN$. 
\State check for conflicts in $solution$
\If{no conflicts found}
    \State \textbf{return} $solution$
\EndIf
\ForAll{$c$ in conflicts}
\State $(\alpha_{i,j}, \alpha_{k,m}) = c$ 
\State $B_1 \gets \neg \alpha_{i,j} \land B_0$
\State $S_1 = $ lowest cost assignment $a_{i,j}$ $\forall$ requests $R_i$ such that $B_1$ is satisfied
\State $B_2 \gets \neg \alpha_{k,m} \land B_0$
\State $S_2 = $ lowest cost assignment $a_{i,j}$ $\forall$ requests $R_i$ such that $B_2$ is satisfied
\State $B_3 \gets \neg \alpha_{i,j} \land \neg \alpha_{k,m} \land B_0$
\State $S_3 = $ lowest cost assignment $a_{i,j}$ $\forall$ requests $R_i$ such that $B_3$ is satisfied
\State $OPEN \gets OPEN \cup (S_1, B_1) \cup (S_2, B_2) \cup (S_3, B_3) $
\EndFor

\EndWhile
\end{algorithmic}
\end{algorithm}

\subsection{Understanding the weaknesses of each algorithm}
\begin{figure}[!b]
    \centering
    \includegraphics[width=8cm]{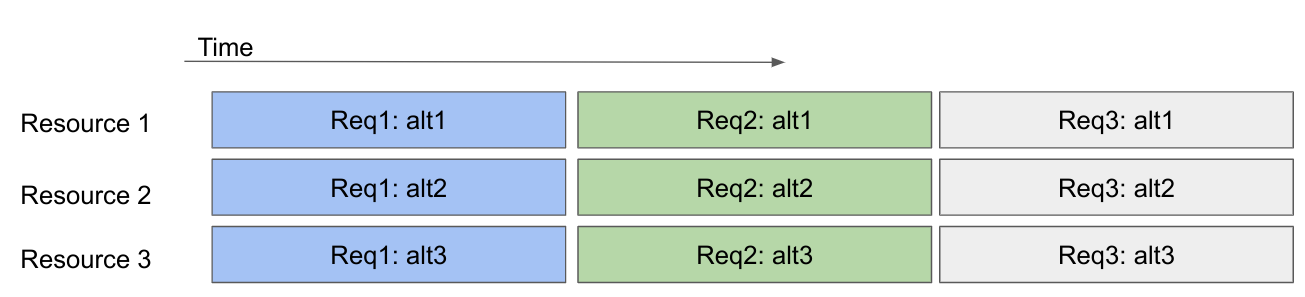}
    \caption{Example of a scenario which elicits a worst case response in the SAT optimization process but is quickly solved by the greedy process in $O(n)$}
    \label{fig:satdevil}
\end{figure}

\begin{figure}
    \centering
    \includegraphics[width=8cm]{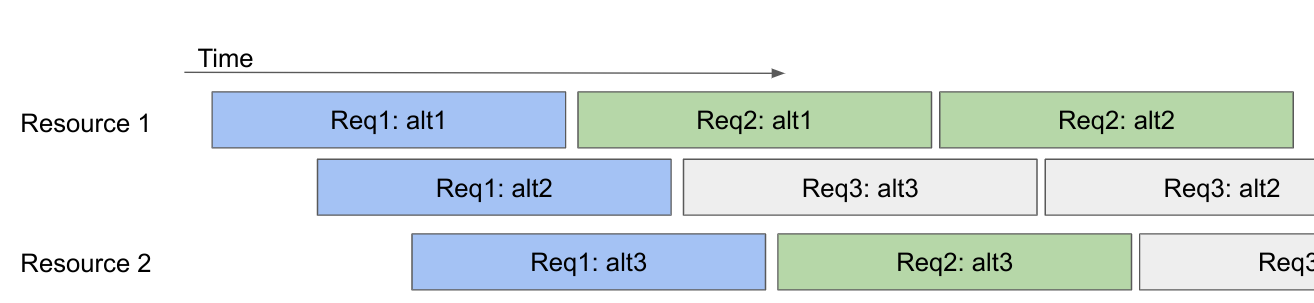}
    \caption{Example of a scenario which elicits a worst case response in the Greedy process but is quickly solved by the SAT optimization process. Assume that alt3 has lowest cost.}
    \label{fig:greedydevil}
\end{figure}

Algorithm \ref{alg:heur} and Algorithm \ref{alg:greedy} are interesting duals of each other when we consider the worst case performance of each algorithm. Fig. \ref{fig:satdevil} shows an adversarial example for Algorithm \ref{alg:heur}. Here, the solver has many alternatives to choose from. Depending on where the first iteration lands, the worst case scenario is that the SAT solver starts from the highest cost request and slowly works its way down through every possible feasible alternative. In contrast, Algorithm \ref{alg:greedy} would find the solution at its first search node.

Fig. \ref{fig:greedydevil} shows an example which would make the greedy algorithm (Algorithm \ref{alg:greedy}) have trouble producing a solution since the overlapping times for the alternatives create numerous conflicts and therefore many branches.

In an extreme scenario where no feasible assignments can be found to simultaneously satisfy all requests, the greedy algorithm would end up doing an exhaustive search, branching at every conflict before determining that no solution exists. On the other hand, the SAT solver in Algorithm \ref{alg:heur} would return UNSAT in the very first iteration.

It is also possible to construct an adversarial example with a feasible solution. One could create a problem where the only viable solutions have high costs, and then sprinkle in a bunch of lower cost alternatives that all lead to irreconcilable conflicts. This would cause the greedy algorithm to examine the lower cost alternatives first, only to discover that they are infeasible. The SAT based solver on the other hand would quickly identify a feasible solution.

The incremental nature of Algorithm \ref{alg:greedy} makes it appealing for ad hoc request scenarios because it can rapidly find a feasible solution and then converge towards the optimal solution with however much time the overall system is willing to budget. When the multi-robot system needs to handle unknown requests coming in at unpredictable times, having any viable solution within a time frame that keeps the overall system responsive is more important than eventually finding the optimal solution. In our benchmarks, we found that the SAT method arrives at the first feasible (but sub-optimal) solution within less than a second for 40 requests with 40 alternatives. Fig. \ref{fig:ttfs} shows the time to first solution of the SAT solver on a Thinkpad P50 with 64GB RAM and a 16-core 11th Gen Intel(R) Core(TM) i7-11850H @ 2.50GHz. This shows that the algorithm is well suited for ad hoc requests.

\begin{figure}[!b]
    \centering
    \includegraphics[width=8cm]{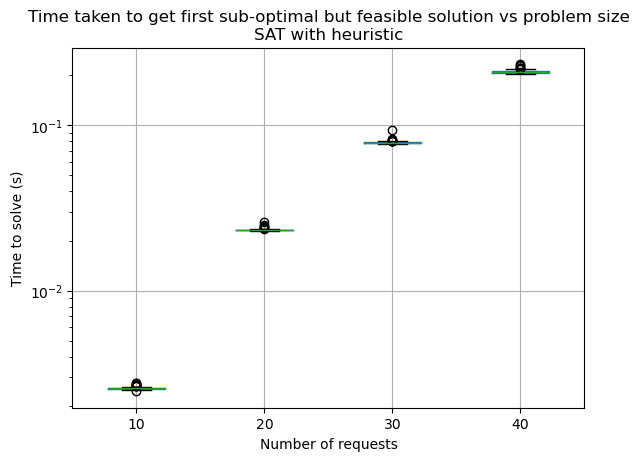}
    \caption{Box-plot of time taken for SAT with heuristics to generate first feasible but suboptimal solution. For a request size of $n$, there are $n$ alternatives. I.E for 40 requests there are 40 alternatives each.}
    \label{fig:ttfs}
\end{figure}

\subsection{Experimental Benchmarks}

We created multiple benchmark sets based on the scenarios described in the section \ref{sec:alt_soln}. They were run single-threaded on a 24-core AMD EPYC 7B12 processor with 96GB of RAM running in Google Cloud. The algorithms were implemented in Rust \cite{matsakis2014rust} and the SAT solver used was the VariSAT solver \cite{Varisat}. The first set of benchmarks were based on a scenario where there are many feasible solutions. This results in a more optimization heavy workload. In such instances, the greedy method reaches optimality faster than repeated application of SAT with heuristics. As shown in Fig. \ref{fig:greedy-label}, as the problem size increases, the performance of the SAT method with heuristics falls far behind the greedy method for problems with simple solutions. Conversely, as the number of conflicts increases, the greedy method fairs worse even with a small number of conflicts (see Fig. \ref{fig:by-conflict}) while the SAT based methods are able to solve large instances in a reasonable amount of time. Fig. \ref{fig:num-requests} shows that the performance of the Greedy Method significantly degrades much faster than SAT with number of conflicts.

\begin{figure}[!t]
    \centering
    \includegraphics[width=8cm, trim=0.5cm 0 0 0, clip]{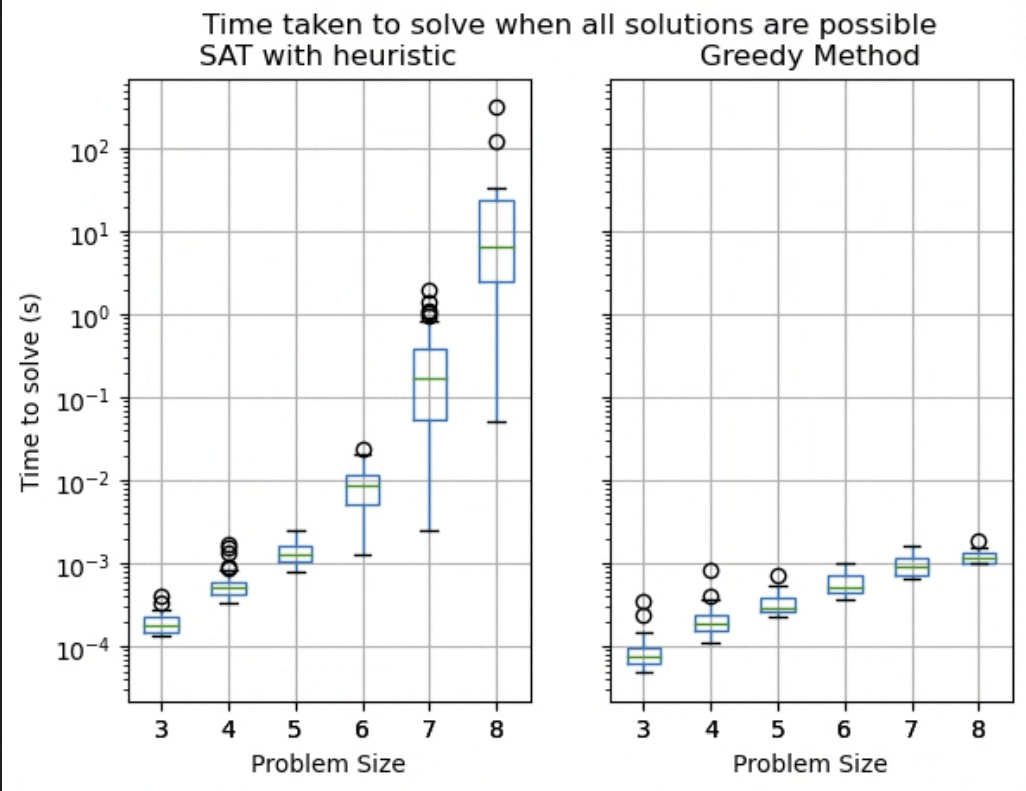}
    \caption{Box plot showing performance comparison when many options exist and no conflicts exist. Problem size $n$ refers to number of requests. Each request also has $n$ alternatives }
    \label{fig:greedy-label}
\end{figure}

\begin{figure}[!b]
    \centering
    \includegraphics[width=8cm, trim=0 0 0.5cm 0, clip]{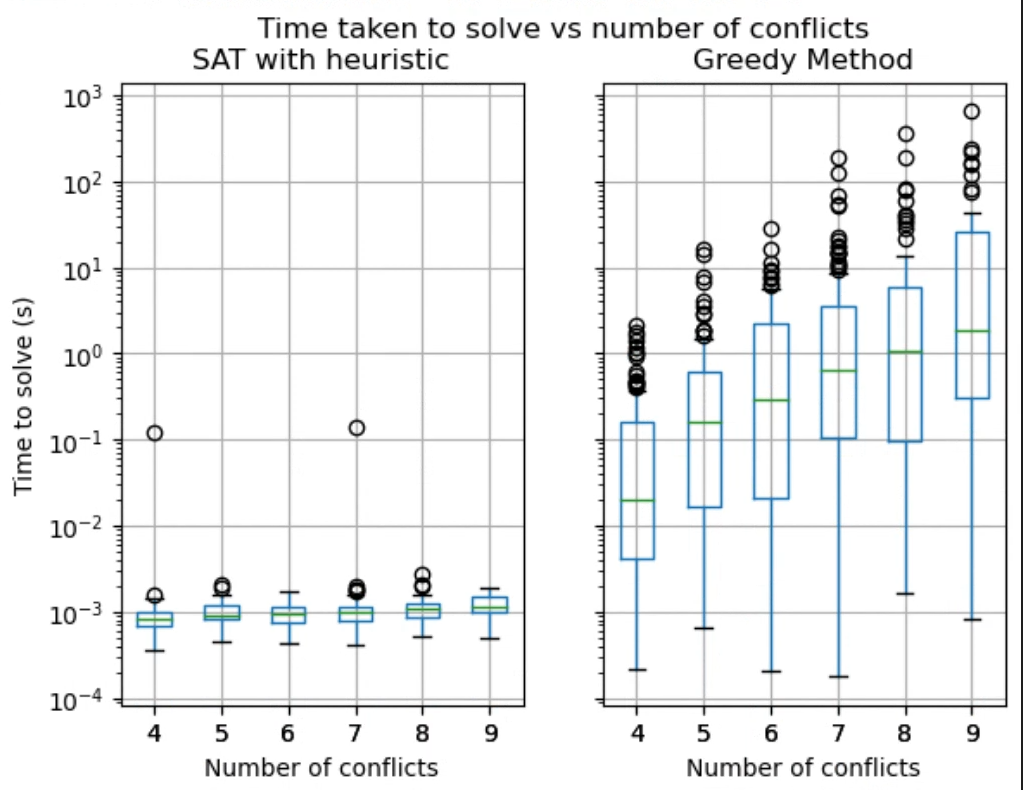}
    \caption{Box plot of performance vs number of conflicts. Here there were 10 requests submitted each with 5 alternatives.}
    \label{fig:by-conflict}
\end{figure}

\begin{figure}[!t]
    \centering
    \includegraphics[width=8cm]{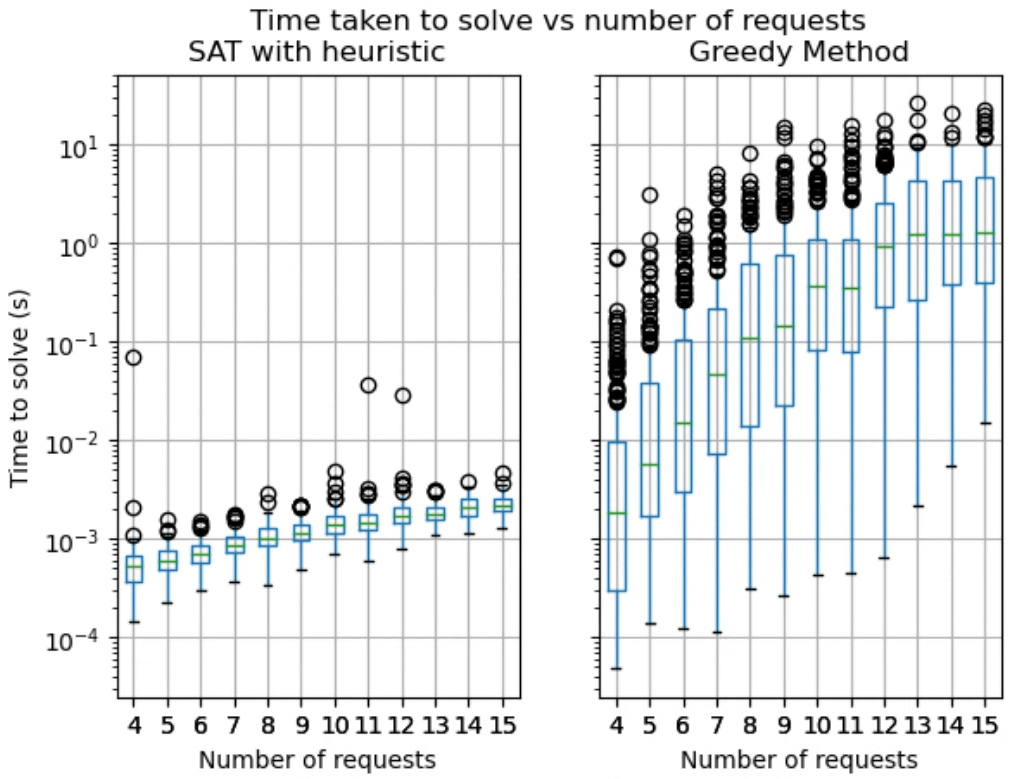}
    \caption{Box plot of performance comparison vs requests in scenario with fixed number of conflicts. Here there were 10 alternatives per option.}
    \label{fig:num-requests}
\end{figure}

\section {Extending the system to encode time ranges}\label{sec:flexitime} 

In a reservation system for robots, it is important to handle scenarios that would require robots to take turns using the same resource. The simple fixed time scheme does not lend itself to solving such problems. For example if two robots each place a request to use the same resource at the same time with no other alternatives, then the fixed time scheme would simply return UNSAT.

This means that we should support a time window in which each reservation alternative may begin. For this, there are two possible solutions. The first is to use the formulation in section \ref{discrete_formulation} and generate a sequence of alternatives for each resource at discrete time intervals. The second option would be to encode ordering into the SAT formulation. The latter has the advantage that a sub-optimal solution can be found very quickly, and an unsatisfiable scenario can be found more quickly thanks to the fact that the starting time range not affecting the formula length (see Section \ref{sec:comparison} for more details). Additionally, there is no need to tune the time resolution to which we discretize.

\subsection {CNF encoding of scheduling constraints to create a model}

We can extend our previous formulation to add support for varying start times. Let's revise our original problem statement.

\subsubsection{Definitions}
\begin{itemize}
    \item Each request can be defined as $R_i$.
    \item A request has $n$ alternatives $\beta_{i,j} = (s_{i,j},l_{i,j},r_{i,j},d_{i,j}, c_{i,j})$ where $ s_{i,j},l_{i,j},d_{i,j} \in \mathbb{R}$. Each alternative is made of an earliest start time $s$, a latest start time $l$, a resource $r$, a cost $c$, and a duration $d$.
    \item An alternative $\beta_{i,j}$ cannot co-exist with alternative $\beta_{k,m}$ if it shares the same resource (i.e. $r_{i,j} = r_{k,m}$) and the earliest end time of $\beta_{i,j}$ (given by  $s_{i,j} + d_{i,j}$) is greater than the latest start time of $\beta_{k,m}$ (given by $l_{k,m}$).
\end{itemize}
The CNF decomposition is similar to the previous decomposition in \eqref{c1} and \eqref{c2}, except we introduce a new set of variables $X_{ijkm}$ to impose ordering. We say $X_{ijkm}$ is true if $\beta_{i,j}$ starts after $\beta_{k,m}$. Based on this definition we will imply a strict total order on all $X_{ijkm}$ within a given resource. Strict total orders generally have four properties: irreflexivity, anti-symmetry, transitivity and connectedness \cite{2011aise}. We do not need to consider irreflexivity as that is something that a SAT solver already encodes. To encode anti-symmetry we start with the clause \eqref{ord_const}:
\begin{equation}
    (x_{ij} \land x_{km} \land X_{ijkm}) \implies \neg X_{kmij} \label{ord_const}
\end{equation}
We can transform \eqref{ord_const} into:
\begin{equation}
   \neg((x_{ij} \land x_{km} \land X_{ijkm})) \lor \neg X_{kmij} 
\end{equation}
Which by De'Morgan's Law simplifies to:
\begin{equation}
    \neg x_{ij} \lor \neg x_{km} \lor \neg X_{ijkm} \lor \neg X_{kmij} \label{ord_cnf}
\end{equation}
We also want to encode transitivity so we assume:
\begin{equation}
  ( X_{ijkm} \land X_{kmnl}) \implies  X_{ijnl} 
\end{equation}
Which gets encoded into
\begin{equation}
  \neg X_{ijkm} \lor \neg X_{kmnl} \lor  X_{ijnl} \label{ord_end}
\end{equation}

We only need connectedness if both $x_{ij}$ and $x_{km}$ are within the same resource and both $x_{ij}$ and $x_{km}$ are true. This can be given by the boolean formula below: 
\begin{equation}
   (x_{ij} \land x_{km}) \implies  ( X_{ijkm} \lor X_{kmij})
\end{equation}
Which simplifies to: 
\begin{equation}
   (\neg x_{ij} \lor \neg  x_{km} \lor X_{ijkm} \lor X_{kmij})
\end{equation}

Finally we add all known literals of $ X_{ijkm}$. The cases where we know $ X_{ijkm}$ are
\begin{itemize}
    \item Iff $\beta_{i,j}$ cannot be scheduled after $\beta_{k,m}$ (i.e. $l_{i,j} < s_{k,m} + d_{k,m} \land s_{i,j} + d_{i,j} > s_{k,m}$), then we add the literal $\neg X_{ijkm}$.
\end{itemize}
  Note that this set of constraints only provides us with a necessary condition. They are not sufficient to prove that an assignment is valid. One still needs to traverse the order determined by the SAT solver and find if the order is valid. In particular our SAT encoding does not take duration into account outside of the indefinite case. Here we can use learning to identify new clauses as defined by Algorithm \ref{alg:clauselearn}. These new clauses will allow us to constrain the output of the SAT solver to only include assignments that satisfy the start time constraint of each alternative while taking their durations into account.

  The key idea in Algorithm \ref{alg:clauselearn} is that certain combinations of assignments would violate the latest start time constraints of one or more alternatives. Using the $X_{ijkm}$ values, we topologically sort the alternatives that would use the same resource within a candidate solution, and then we calculate what the concrete start time values would be for that combination of assignments. The candidate solution becomes invalid if any of the assignments has been pushed out beyond its latest start time constraint. In this case we retrace which assignments were responsible for causing the unacceptable delay and mark that combination as illegal. This allows us to further limit the search space such that we converge on a valid solution.

\begin{algorithm}[!t]
\caption{Conflict detection and clause learning for Resource Ordering}\label{alg:clauselearn}
\begin{algorithmic}

\State let $G(V,E)$ be the graph formed where $V$ is the set of all the $x_{ij}$ assigned to be $true$ and in resource $r*$, and  $E$ is the set of all true $X_{ijkm}$.
\State $sorted\_assignment \gets topological\_sort(G)$ 
\State schedule $\gets \{\}$
\State banned\_clauses $\gets \{\}$
\ForAll{$i$ in sorted\_assignment.len()}
  \State $s_i =$ sorted\_assignment[i].earliest\_start 
  \State $l_i =$ sorted\_assignment[i].latest\_start 
  \If{$s_i >$ schedule.last.end\_time() }
    
    \State schedule.insert\_at($s_i$, sorted\_assignment[i])
  \ElsIf{$l_i >$ schedule.last.end\_time()}
    \State $t \gets$ schedule.last.end\_time()
    \State schedule.insert\_at($t$, sorted\_assignment[i])
  \Else
    \State banned\_clause $\gets ()$
    \State $j =$ schedule.len()-1
    \While {schedule[j].start\_time() == schedule[j-1].end\_time()} 
      \State banned\_clause $\lor= \neg$ assignment[i]
      \State $j \gets j-1$
    \EndWhile
    \State banned\_clauses.push(banned\_clause)
  \EndIf
\EndFor
\State \Return banned\_clauses
\end{algorithmic}
\end{algorithm}

\subsection {Comparing Discretization vs Continuous representation}\label{sec:comparison}

An alternative to the formulation described in equations \eqref{ord_cnf}-\eqref{ord_end} is to discretize each order and check for satisfiability using the conditions described in \eqref{c3}. The first thing that merits discussion is the order of growth of clauses. In the case of discretization, we find that the number of literals is proportional to the interval of each alternative ($t$) multiplied by the number of resources to consider ($n$). The order of growth of the clauses is then $O(n^2t^2)$. On the other hand, the continuous representation is independent of $t$. However the transitivity constraints result in $O(n^2)$ literals and $O(n^3)$ clauses. Discretization leads to the added problem that assigned reservations may have gaps between them. In practice this can be fixed with a post processing pass.

Fig. \ref{fig:discrete_vs_continuous.png} shows the start time range has a significant impact on the time to solve a fixed-timestep discretization based solver, whereas the continuous solver is able to find solutions several orders of magnitudes faster and is virtually unaffected by the start time range.

\begin{figure}[!t]
    \centering
    \includegraphics[width=8cm]{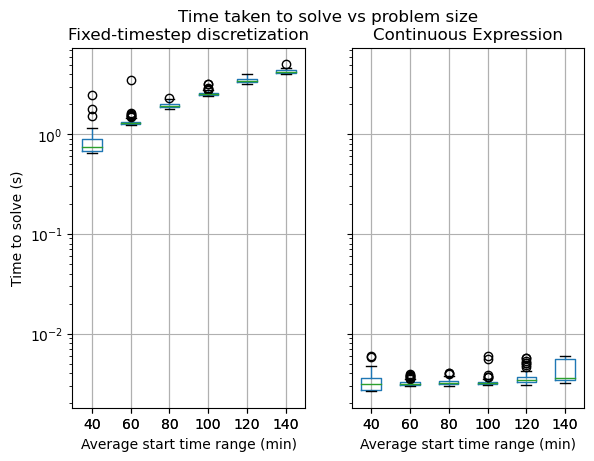}
    \caption{Discrete vs continuous expression: 5 requests with 5 alternatives were given. The discretiation time step was 10 minutes. The algorithms were run till the first feasible alternative is found}
    \label{fig:discrete_vs_continuous.png}
\end{figure}

Note that for cost optimization Algorithm \ref{alg:heur} still applies with some limitations. The cost cannot depend on time when using the continuous formulation, whereas the discretized time representation can encode different costs into alternatives that start at different times.

\section {Validation\label{sec:val}}


\begin{figure}[!b]
    \centering
    \includegraphics[width=8cm]{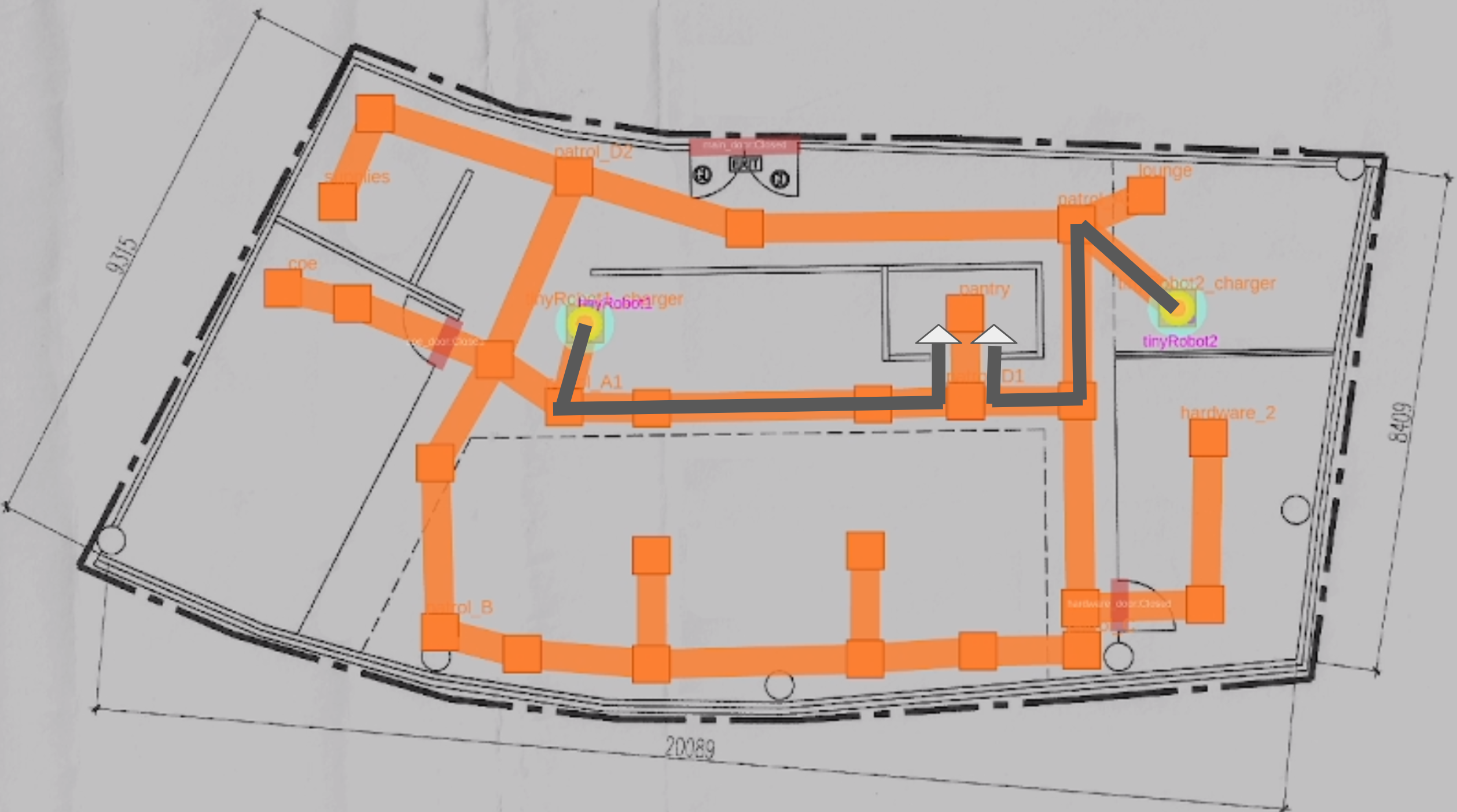}
    \caption{Example of simulated world. The arrows show requested trajectory in an example scenario. The reservation system would re-order the requests such that no two robots would conflict with each other for use of the pantry.}
    \label{fig:sim-world}
\end{figure}

We ran several scenarios in Gazebo \cite{Gazebo} for 24 hours in simulated time to check that no deadlocks occur. The first scenario involved 2 robots periodically making a request to the same location to simulate ``charging''. Fig. \ref{fig:sim-world} shows an example of conflicting requests made. We tried similar simulations with 3 robots in the same world. The system would automatically re-order the robots such that no conflict occurs. We use Open-RMF's office world from the $\texttt{rmf\_demos}$ package to validate our algorithm's behaviour. We also had scenarios where random requests were made for robots to go to different waypoints. The algorithms and test harnesses were implemented in ROS 2 \cite{ros2} with Rust \cite{matsakis2014rust} and can be found on GitHub\footnote{\url{https://github.com/open-rmf/rmf_reservation}}.

Additionally, we tested our algorithm for 30 minutes on two Smorphi robots\footnote{\url{https://www.wefaarobotics.com}} shown in Fig. \ref{fig:physical}. In this test, the two robots repeatedly requested the same spot. The robot would request the resource (in this case the position in the center of the L bend), claim its assignment from the reservation system, and then proceed to the position when permitted. In order to ensure that no other robot can move in to the same spot, the claim held by the robot must be released before any other claimants are permitted to go. The system releases the claim once the robot moves away from the claimed position. The reservation system would then give permission to the next claim. This process of requesting, claiming, waiting for permission, proceeding, and then returning to their start points cycled for 30 minutes without any deadlocks.

\begin{figure}[!t]
    \centering
    \includegraphics[width=8cm]{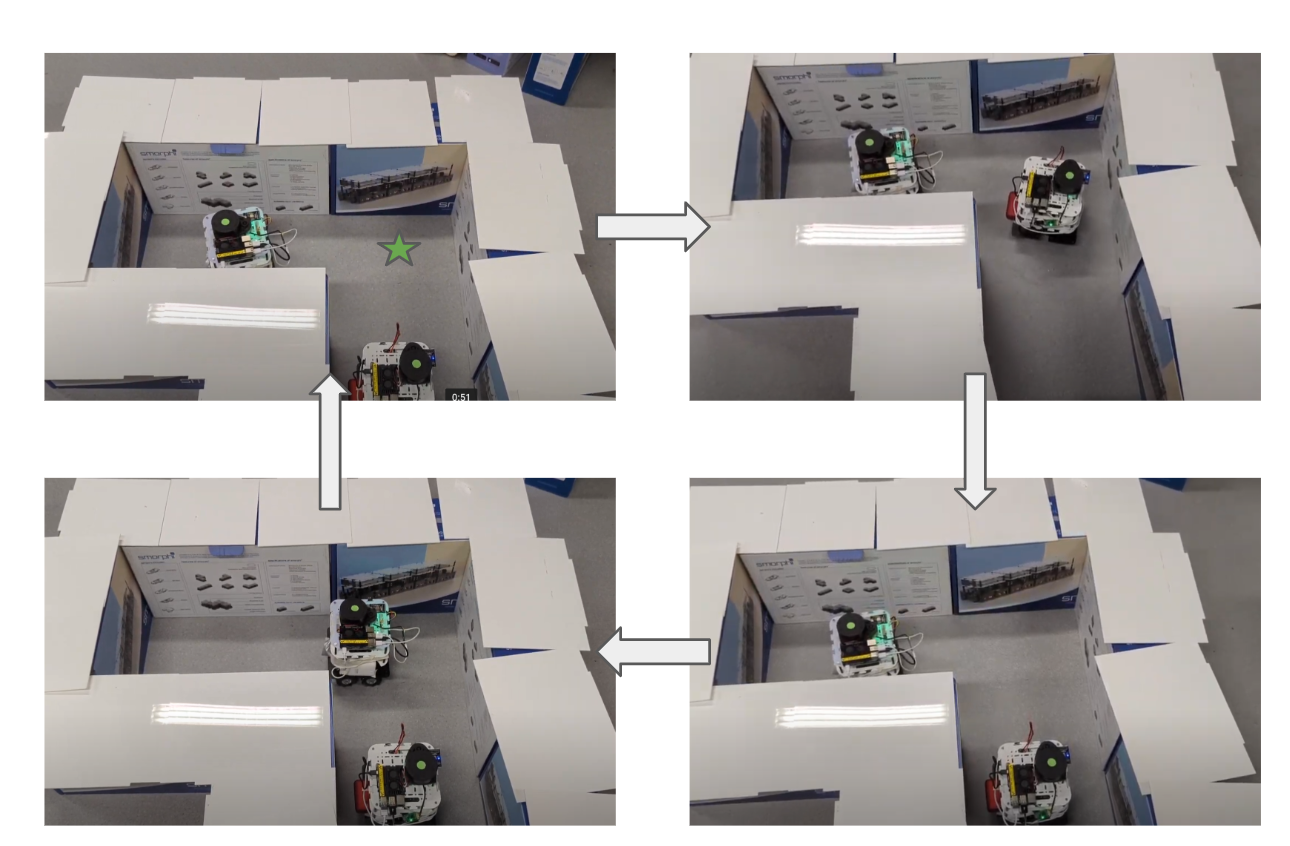}
    \caption{Example of physical demo using two Smorphi robots. The green star shows the contended resource. 
    }
    \label{fig:physical}
\end{figure}

\section{Conclusions and Future Work\label{sec:conclusion}}
This paper proposed a novel algorithm for an ad-hoc reservation system that optimizes the assignment of requested resources while accounting for the individual costs to the requesters. We race a SAT solver against a Greedy solver since the performance profile of each method is complementary. Additionally we have shown how to encode order into the boolean satisfiability problem and shown how the constrained assignment problem can be optimized via repeated application of a SAT solver.

Our benchmarks, simulations and real world demonstration show the feasibility of using the SAT based algorithms for resource sharing. Of particular importance are the heuristics developed for assignment and the encoding developed for total ordering. 

As future work, we would like to examine how the costs of assignments produced by the algorithm in Section \ref{sec:flexitime} can depend on time or some other arbitrary function or are we forced to use discretization to account for the cost of delaying an assignment. We also only look at discretization at fixed time steps, however we may be able to minimize computation time if were smart about the way we sampled possible assignments. Other interesting ideas include extending the CNF formulations here to more complex planning schemes where one reservation might depend on which of the previous reservations was awarded. Finally are there ways of automatically discovering theorems that can rapidly guide optimizations.

\addtolength{\textheight}{-12cm}   




\bibliographystyle{IEEEtran}
\bibliography{IEEEabrv,references}

\end{document}